\documentclass{article} %
\usepackage[utf8]{inputenc}
\usepackage[T1]{fontenc}

\usepackage{amssymb,amsmath,amsthm,mathtools}
\usepackage[unicode=true]{hyperref}       %
\usepackage{cleveref}
\usepackage{bbm}
\usepackage{xifthen}
\usepackage{todonotes}
\usepackage{libertine} %

\usepackage{style}
\usepackage[round,authoryear]{natbib}
\usepackage[letterpaper,lmargin=1.5in,rmargin=1.5in,tmargin=1in,bmargin=1in]{geometry} %

\usepackage[shortlabels]{enumitem}
\setlist{topsep=1pt,parsep=1pt,partopsep=1pt,itemsep=1pt}

\usepackage{wrapfig}
\usepackage{graphicx}
\usepackage{adjustbox}
\usepackage{booktabs}
\usepackage{amssymb}
\usepackage{enumitem}
\usepackage{multicol}
\usepackage{subcaption}
\usepackage{wrapfig}

\usepackage{pifont}

\usepackage{xspace}
\usepackage{booktabs}
\usepackage{graphicx}
\usepackage{fancyvrb}
\usepackage{cleveref}
\usepackage{nicematrix}
\usepackage{nicefrac}
\usepackage{microtype}

\usepackage[ruled,linesnumbered]{algorithm2e}
\makeatletter
\newcommand{\algorithmstyle}[1]{\renewcommand{\algocf@style}{#1}}
\makeatother

\title{\mbox{Classifier-Free Guidance is a Predictor-Corrector}}

\author{
{\bf Arwen Bradley$^*$}\\
Apple
\and
{\bf Preetum Nakkiran$^*$}\\
Apple 
}
\date{}

\begin{document}

\maketitle
\def\thefootnote{*}\footnotetext{Equal contribution}\def\thefootnote{\arabic{footnote}}

\begin{abstract}
We investigate the theoretical foundations of
classifier-free guidance (CFG).
CFG is the dominant method of conditional sampling for text-to-image diffusion models, yet
unlike other aspects of diffusion, it remains on shaky theoretical footing. In this paper, we first disprove common misconceptions,
by showing that CFG interacts differently with
DDPM~\citep{ho2020denoising} and DDIM~\citep{song2021denoising},
and neither sampler with CFG generates the gamma-powered distribution $p(x|c)^\gamma p(x)^{1-\gamma}$. 
Then, we clarify the behavior of CFG by showing that it is a kind of
predictor-corrector method \citep{song2020score} that alternates between denoising and sharpening, which we call predictor-corrector guidance (PCG).
We prove that in the SDE limit,
CFG is actually equivalent to
combining a DDIM predictor for the conditional distribution
together with a Langevin dynamics corrector
for a gamma-powered distribution
(with a carefully chosen gamma).
Our work thus provides a lens to theoretically understand CFG
by embedding it in a broader design space
of principled sampling methods.

\end{abstract}

\section{Introduction}

Classifier-free-guidance (CFG) has become an essential part of modern
diffusion models, especially in text-to-image applications \citep{dieleman2022guidance, rombach2022high,nichol2021glide,podell2023sdxl}.
CFG is intended to improve conditional sampling, e.g. generating images conditioned on a given class label or text prompt \citep{ho2022classifier}.
The traditional (non-CFG) way to do conditional sampling
is to simply train a model for the conditional distribution $p(x \mid c)$,
including the conditioning $c$
as auxiliary input to the model.
In the context of diffusion, this means training a model
to approximate the conditional score
$s(x, t, c) := \grad_x \log p_t(x \mid c)$
at every noise level $t$, and sampling from this model
via a standard diffusion sampler (e.g. DDPM).
Interestingly, this standard way of conditioning usually does not perform well for 
diffusion models, for reasons that are unclear.
In the text-to-image case for example,
the generated samples tend to be visually incoherent and not faithful to the prompt,
even for large-scale models \citep{ho2022classifier,rombach2022high}.

Guidance methods, such as CFG and its predecessor classifier guidance \citep{sohl2015deep, song2020score, dhariwal2021diffusion},
are methods introduced to improve the quality of conditional samples.
During training, CFG requires learning a model for both the 
unconditional and conditional scores 
($\grad_x \log p_t(x)$ and $\grad_x \log p_t(x|c)$).
Then, during sampling, CFG runs
any standard diffusion sampler (like DDPM or DDIM),
but replaces the true conditional scores with
the ``CFG scores''
\begin{align}
\wt{s}(x, t, c) := \gamma\grad_x \log p_t(x \mid c)  + (1-\gamma) \grad \log p_t(x), 
\label{eq:cfg_score_def}
\end{align}
for some $\gamma > 0$. 
This turns out to produce much more coherent samples in practice,
and so CFG is used in almost all 
modern text-to-image diffusion models \citep{dieleman2022guidance}.
A common intuition for why CFG works
starts by observing that Equation~\eqref{eq:cfg_score_def} is the score of a \emph{gamma-powered} distribution:
\begin{align}
p_{t,\gamma}(x|c) 
&:= p_t(x)^{1-\gamma} p_t(x|c)^\gamma, \label{eq:p_gamma_cfg}
\end{align} which is also proportional to $p_t(x) p_t(c|x)^\gamma$. Raising $p_t(c|x)$ to a power $\gamma > 1$ sharpens the classifier around its modes, thereby emphasizing the ``best'' exemplars of the given class or other conditioner at each noise level. Applying CFG --- that is, running a standard sampler with the usual score replaced by the CFG score at each denoising step --- is supposed to increase the influence of the conditioner on the final samples.

\begin{figure}[t]
    \centering
    \includegraphics[width=1.0\textwidth]{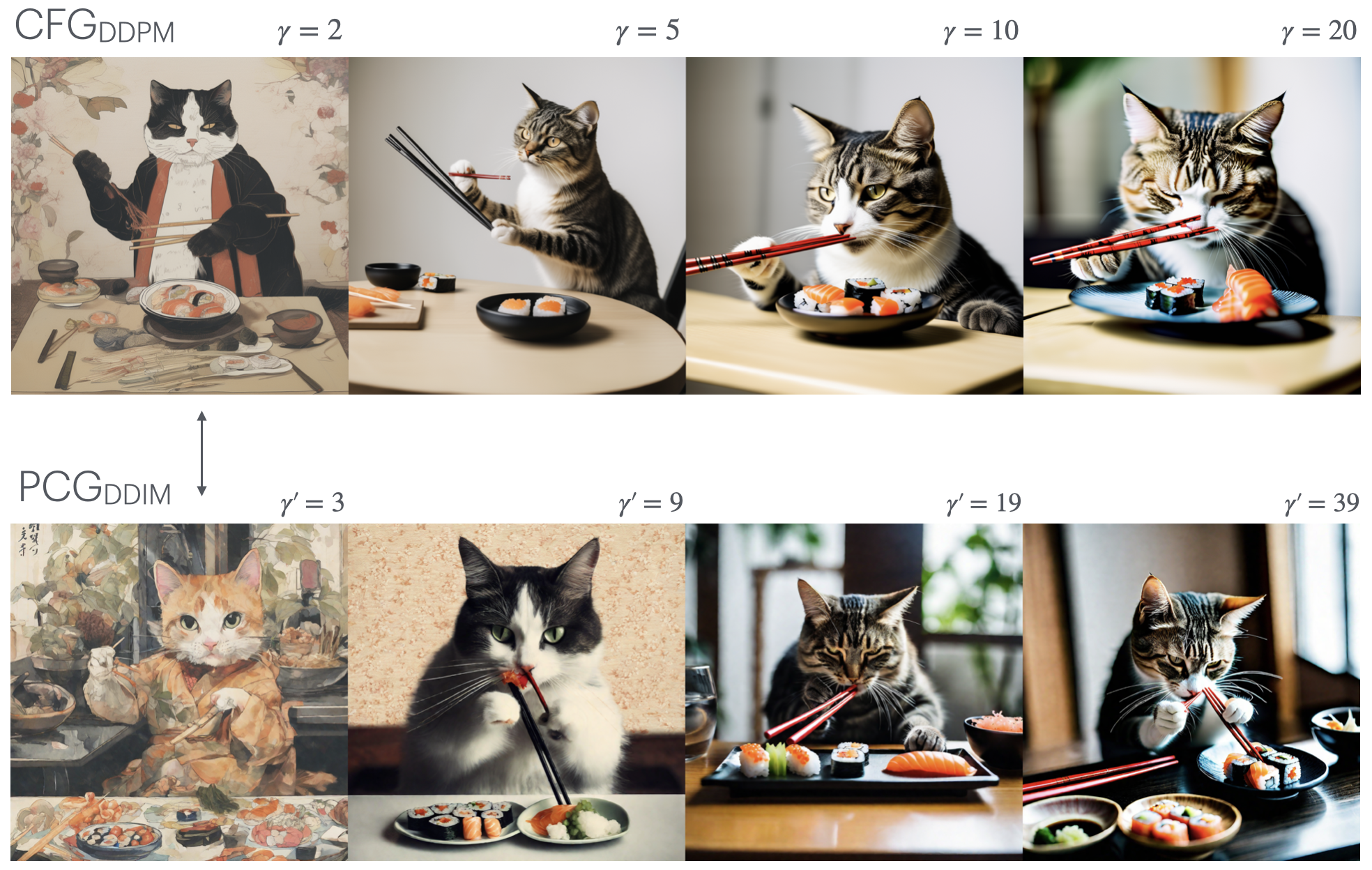}
    \caption{{\bf CFG vs. PCG}. We prove that the DDPM variant of classifier-free guidance (top) is equivalent
    to a kind of predictor-corrector method (bottom),
    in the continuous limit.
    We call this latter method
    ``predictor-corrector guidance'' (PCG), defined
    in Section~\ref{sec:pcg_warmup}.
    The equivalence holds for all
    CFG guidance strengths $\gamma$, with
    corresponding PCG parameter
    $\gamma'=(2\gamma-1)$,
    as given in Theorem~\ref{thm:main}.
    Samples from SDXL with prompt:
    ``photograph of a cat eating sushi using chopsticks''.
    }
    \label{fig:cfg_vs_pcg}
\end{figure}

However, CFG does not inherit the theoretical correctness guarantees
of standard diffusion, because the CFG scores do not necessarily correspond to 
a valid diffusion forward process.
The fundamental issue (which is known, but still worth emphasizing) is that $p_{t,\gamma}(x|c)$ is not the same as the distribution obtained by applying a forward diffusion process to the gamma-powered data distribution $p_{0,\gamma}(x|c)$.
That is, letting $\Nt{t}{p}$ denote the distribution produced by starting from a distribution $p$ and running the diffusion forward process up to time $t$, we have
$$p_{t,\gamma}(x|c) := N_t[p_0(x|c)]^\gamma \cdot N_t[p_0(x)]^{1-\gamma} \neq N_t\left[p_0(x|c)^\gamma p_0(x)^{1-\gamma}\right].$$
Since the distributions $\{p_{t,\gamma}(x|c)\}_t$ \emph{do not correspond to any known forward diffusion process}, we cannot properly interpret the CFG score \eqref{eq:cfg_score_def} as a denoising direction; and using the CFG score in a sampling loop like DDPM or DDIM  is no longer theoretically guaranteed to produce a sample from $p_{0,\gamma}(x|c)$ or any other known distribution.
Although this flaw is known in theory (e.g. \citet{du2023reduce,karras2024guiding}),
it is largely ignored in practice and in much of the literature.
The theoretical foundations of CFG are thus unclear, and important questions remain open.
Is there a principled way to think about why CFG works? And what does it even mean for CFG to ``work'' -- what problem is CFG solving?
We make progress towards understanding the foundations of CFG,
and in the process we uncover several new aspects and 
connections to other methods.

\begin{enumerate}
    \item First, we disprove common misconceptions about CFG by counterexample. We show that the DDPM and DDIM variants of CFG can generate different distributions, neither of which is the gamma-powered data distribution $p_0(x)^{1-\gamma} p_0(x|c)^\gamma$.

    \item We define a
    family of methods called predictor-corrector guidance (PCG),
    as a natural way to approximately sample
    from gamma-powered distributions.
    PCG alternates
    between denoising steps and Langevin dynamics steps.
    Unlike typical predictor-corrector methods \citep{song2020score}, in PCG
    the corrector operates on a different (sharper) distribution
    than the predictor.

    \item We prove that in the continuous-time limit,
    CFG is equivalent to PCG with a careful choice of parameters.
    This gives a principled way to interpret CFG:
    it is implicitly an annealed Langevin dynamics.

    \item
    For demonstration purposes, we implement the PCG sampler
    for Stable Diffusion XL and observe that it produces samples
    qualitatively similar to CFG, with guidance scales determined
    by our theory.
    Further, we explore the design axes exposed by the PCG framework, namely
    guidance strength and Langevin parameters, in order to clarify their respective effects.
\end{enumerate}

\section{Preliminaries}
We adopt the continuous-time
stochastic differential equation (SDE) formalism of diffusion
from \citet{song2020score}. These continuous-time results can be translated to discrete-time algorithms; we give explicit algorithm descriptions for our experiments.

\subsection{Diffusion Samplers}
Forward diffusion processes start with a conditional data distribution $p_0(x|c)$ and gradually corrupt it with Gaussian noise, with $p_t(x|c)$ denoting the noisy distribution at time $t$. 
The forward diffusion runs up to a time $T$ large enough that $p_T$ is approximately pure noise. To sample from the data distribution, we first sample from the Gaussian distribution $p_T$ and then run the diffusion process in reverse (which requires an estimate of the score, usually learned by a neural network). A variety of samplers have been developed to perform this reversal. DDPM \citep{ho2020denoising} and DDIM \citep{song2021denoising} are standard samplers that correspond to discretizations of a reverse-SDE and reverse-ODE, respectively. Due to this correspondence, we refer to the reverse-SDE as DDPM and the reverse-ODE as DDIM for short.
We will mainly consider the \emph{variance-preserving} (VP) diffusion process 
from \citet{ho2020denoising}, although most of our discussion applies equally to other settings (such as variance-exploding). The forward process, reverse-SDE, and equivalent reverse-ODE for the VP conditional diffusion are \citep{song2020score}
\begin{align}
\textsf{Forward SDE}: dx &= -\half \beta_{t} x dt + \sqrt{\beta_{t}} dw. \label{eq:ddpm_sde} \\
\textsf{DDPM SDE}: \quad
dx &=
-\half \beta_{t} x ~dt 
- \beta_t \grad_x \log p_t(x|c) dt
+ \sqrt{\beta_{t}} d\bar{w}  \label{eq:ddpm} \\
\textsf{DDIM ODE}: \quad
dx &= 
-\half \beta_{t} x~dt 
- \half \beta_t \grad_x \log p_t(x|c) dt. \label{eq:ddim}
\end{align}
The unconditional version of each sampler simply replaces $p_t(x|c)$ with $p_t(x)$.
Note that the \emph{score} $\grad_x \log p_t(x|c)$ appears in both \eqref{eq:ddpm} and \eqref{eq:ddim}. Intuitively, the score points in a direction toward higher probability, and so it helps to reverse the forward diffusion process. The score is unknown in general, but can be learned via standard diffusion training methods.

\subsection{Classifier-Free Guidance}

CFG replaces the usual conditional score $\nabla_x \log p_t(x|c)$ in~\eqref{eq:ddpm} or~\eqref{eq:ddim} at each timestep $t$ with the alternative score $\grad_x \log p_{t, \gamma}(x|c)$.
In SDE form, the CFG updates are
\begin{align}
\textsf{$\cfg_\ddpm$}: \quad
dx &=
- \half \beta_{t} x ~dt 
- \beta_t \grad_x \log p_{t, \gamma}(x|c) dt
+ \sqrt{\beta_{t}} d\bar{w} \label{line:ddpm-cfg} \\
\textsf{$\cfg_\ddim$}: \quad
dx &= 
-\half \beta_{t} x~dt 
- \half \beta_t \nabla \log p_{t, \gamma}(x|c) dt \label{line:ddim-cfg}, \\
\text{where } \grad_x \log p_{t, \gamma}(x|c) &= (1-\gamma)\nabla_x \log p_t(x) + \gamma \nabla_x \log p_t(x | c). \notag
\end{align}

\subsection{Langevin Dynamics}
Langevin dynamics \citep{rossky1978brownian,parisi1981correlation}
is another sampling method, which starts from an arbitrary initial distribution and iteratively transforms it into a desired one.
Langevin dynamics (LD) is given by the following SDE \citep{robert1999monte}
\begin{align}
    dx &= \frac{\epsilon}{2} \nabla \log \rho(x) dt + \sqrt{\epsilon} dw. \label{eq:ld}
\end{align}
LD converges (under some assumptions)
to the steady-state $\rho(x)$ \citep{roberts1996exponential}. That is, letting $\rho_s(x)$ denote the solution of LD at time $s$, we have $\lim_{s \to \infty} \rho_s(x) = \rho(x)$. Similar to diffusion sampling, LD requires the score of the desired distribution $\rho$ (or a learned estimate of it).
\section{Misconceptions about CFG}
\label{sec:misconceptions}

\begin{figure}
    \centering
    \includegraphics[width=0.49\textwidth]{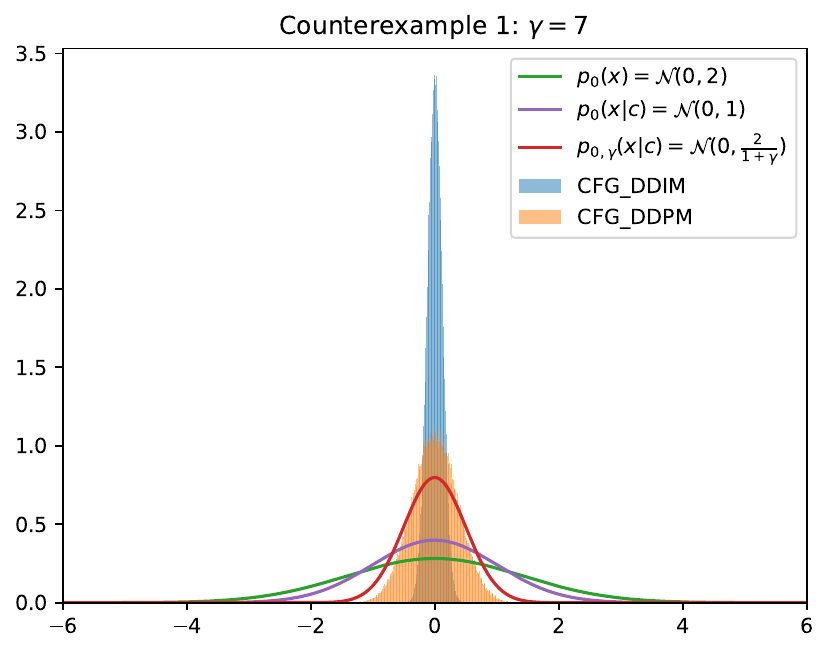}
    \includegraphics[width=0.49\textwidth]{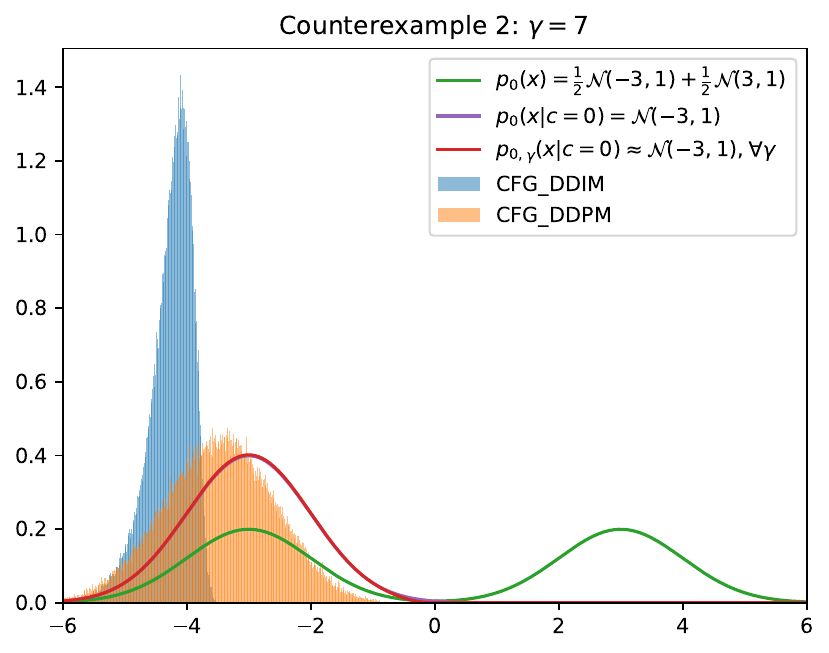}
    \caption{{\bf Counterexamples: $\cfg_\ddim \neq \cfg_\ddpm \neq$ gamma-powered.} 
    $\cfg_\ddim$ and $\cfg_\ddpm$ do not generate the same output distribution,
    even when using the same score function.
    Moreover, neither generated distribution
    is the gamma-powered distribution $p_{0,\gamma}(x|c)$.
    (Left) Counterexample 1 (section~\ref{sec:counterex1}): $\cfg_\ddim$ yields a sharper distribution than $\cfg_\ddpm$, and both are sharper than $p_{0,\gamma}(x|c)$.
    (Right) Counterexample 2 (section~\ref{sec:counterex2}): Neither $\cfg_\ddim$ nor $\cfg_\ddpm$ yield even a scaled version of the gamma-powered distribution $p_{0,\gamma}(x|c) = \cN(-3,1)$. The $\cfg_\ddpm$ distribution is mean-shifted relative to $p_{0,\gamma}(x|c)$. The $\cfg_\ddim$ distribution is mean-shifted and not even Gaussian (note the asymmetrical shape).
    }
    \label{fig:ddpm_vs_ddim_counterexamples}
\end{figure}

\bigskip

We first observe that the exact definition of CFG matters: specifically, the sampler with which it used. Without CFG, DDPM and DDIM generate equivalent distributions. However, we will prove that with CFG, DDPM and DDIM
can generate different distributions, as follows:

\begin{theorem}[DDIM $\ne$ DDPM; informal ] 
There exists a joint distribution $p(x, c)$ over 
inputs $x \in \R$ and conditioning $c \in \R$, such that the following holds.
Consider generating a sample via CFG with conditioning $c=0$,
guidance-scale $\gamma \gg 0$, and using either DDPM or DDIM samplers.
Then, the generated distributions will be approximately
\begin{align}
    \hat{p}_{\mathrm{ddpm}} \approx \cN(0, \gamma^{-1}) ;
    \quad
    \hat{p}_{\mathrm{ddim}} \approx \cN(0, 2^{-\gamma}).
\end{align}
In particular, the DDIM variant of CFG is exponentially sharper than the DDPM variant. 
\end{theorem}

Next, we disprove the misconception that CFG 
generates the gamma-powered distribution data:

\begin{theorem}[CFG $\neq$ gamma-sharpening, informal]
There exists a joint distribution $p(x, c)$ and a $\gamma > 0$
such that neither $\cfg_\ddim$ 
nor $\cfg_\ddpm$ produces the gamma-powered distribution $p_{0,\gamma}(x|c) \propto p_0(x)^{1-\gamma} p_0(x|c)^\gamma$.
\end{theorem}

We prove both claims in the next section using simple Gaussian constructions.

\subsection{Counterexample 1}
\label{sec:counterex1}

We first present a setting that allows us to \emph{exactly} solve the ODE and SDE dynamics of CFG in closed-form, and hence to find the exact distribution sampled by running CFG. This would be intractable in general, but it is possible for a specific problem, as follows.

Consider the setting where $p_0(x)$ and $p_0(x|c=0)$ are both zero-mean Gaussians, but with different variances. Specifically, $(x_0,c)$ are jointly Gaussian, with $p(c) = \cN(0,1)$, $p_0(x|c) = c + \cN(0,1)$. Therefore
\begin{align}
p_0(x) &= \cN(0,2) \notag \\
p_0(x|c=0) &= \cN(0,1) \notag \\
p_{0, \gamma}(x|c=0) &= \cN(0, \frac{2}{\gamma+1})
\label{eq:1d_gauss_problem_1}
\end{align}
For this problem, we can solve $\cfg_\ddim$ \eqref{line:ddim-cfg} and $\cfg_\ddpm$ \eqref{line:ddpm-cfg} analytically; that is, we solve initial-value problems for the reversed dynamics to find the sampled distribution of $\hat x_t$ in terms of the initial-value $x_T$. Applying these results to $t=0$ and averaging over the known Gaussian distribution of $x_T$ gives the exact distribution of $\hat x_0$ that CFG samples. The full derivation is in Appendix \ref{app:prob_1}. The final CFG-sampled distributions are:
\begin{align}
    \cfg_\ddpm: \quad \hat x_0 &\sim \mathcal{N} \left( 0, \frac{2 - 2^{2-2\gamma}}{2\gamma-1} \right) \label{eq:prob_1_ddpm} \\
    \cfg_\ddim: \quad \hat x_0 &\sim \mathcal{N} \left(0, 2^{1-\gamma} \right).
    \label{eq:prob_1_ddim}
\end{align}
This shows that for any $\gamma > 1$, the $\cfg_\ddim$ distribution is sharper than the $\cfg_\ddpm$ distribution, and both are sharper than the gamma-powered distribution $p_{0, \gamma}(x|c=0)$. (Even though the distributions all have the same mean, their different variances make them distinct.) In fact, for $\gamma \gg 1$, the variance of DDPM-CFG is approximately $\frac{2}{2\gamma-1}$, which is about twice the variance of $p_{0, \gamma}(x|c=0)$. In Figure~\ref{fig:ddpm_vs_ddim_counterexamples}, we compare the $\cfg_\ddim$ and $\cfg_\ddpm$ distributions -- sampled using an exact denoiser (see Appendix \ref{app:exact_denoiser}) within DDIM/DDPM sampling loops -- to the unconditional, conditional, and gamma-powered distributions.

\subsection{Counterexample 2}
\label{sec:counterex2}

In the above counterexample, the $\cfg_\ddim$, $\cfg_\ddpm$, and gamma-powered distributions had different variances but the same Gaussian form, so one might wonder whether the distributions differ only by a scale factor in general. This is not the case, as we can see in a different counterexample that reveals greater qualitative differences, in particular a symmetry-breaking behavior of CFG. 

In Counterexample 2, the unconditional distribution is a Gaussian mixture with two clusters with equal weights and variances, and means at $\pm \mu$.
\begin{align}
    c &\in \{0, 1\}, \quad p(c = 0) = \half \notag \\ 
    p_0(x_0 | c=0) &= \cN(-\mu, 1), \quad  p_0(x_0 | c=1) = \cN(\mu, 1) \notag\\
    p_0(x_0) &= \half p_0(x_0 | c=0) + \half p_0(x_0 | c=1)
    \label{eq:1d_gauss_problem_2}
\end{align}
If the means are sufficiently separated ($\mu \gg 1$), then the gamma-powered distribution for $\gamma \ge 1$ is approximately equal to the conditional distribution, i.e.
$ p_{0, \gamma} (x|c) \approx p_0(x|c),$
due to the near-zero-probability valley between the conditional densities
(see Appendix \ref{app:prob_2}). However, for sufficiently high noise the clusters begin to merge, and $p_{t, \gamma} (x|c) \neq p_t(x|c)$. In particular, $p_{0, \gamma} (x|c)$ is approximately Gaussian with mean $\pm \mu$, but $p_{t, \gamma} (x|c) \neq p_t(x|c)$ is not.
Although we cannot solve the CFG ODE and SDE in this case, we can empirically sample the $\cfg_\ddim$ and $\cfg_\ddpm$ distributions using an exact denoiser and compare them to the gamma-powered distribution. In particular, we see that neither $\cfg_\ddim$ nor $\cfg_\ddpm$ is Gaussian with mean $\pm \mu$, hence neither is a scaled version of the gamma-powered distribution. The results are shown in Figure \ref{fig:ddpm_vs_ddim_counterexamples}.
\section{CFG as a predictor-corrector}

The previous sections illustrated the subtlety in understanding CFG.
We can now state our main structural characterization,
that CFG is equivalent to a special kind of \emph{predictor-corrector} method \citep{song2020score}.

\subsection{Predictor-Corrector Guidance}
\label{sec:pcg_warmup}
As a warm-up, suppose we actually wanted to sample from the
gamma-powered distribution:
\begin{align}
\label{eqn:pgamma}
p_{\gamma}(x|c) \propto p(x)^{1-\gamma} p(x|c)^\gamma.
\end{align}
A natural strategy is to run Langevin dynamics w.r.t. $p_\gamma$.
This is possible in theory because 
we can compute the score of $p_\gamma$ from the known scores of $p(x)$
and $p(x\mid c)$:
\begin{align}
\grad_x \log p_\gamma(x \mid c) = (1-\gamma) \grad_x \log p(x)
+ \gamma \grad_x \log p(x \mid c).
\end{align}
However this won't work in practice, due to the well-known issue
that vanilla Langevin dynamics has impractically slow mixing times for
many distributions of interest \citep{song2019generative}.
The usual remedy for this is to use some kind of annealing,
and the success of diffusion teaches us that the diffusion process
defines a good annealing path \citep{song2020score,du2023reduce}.
Combining these ideas yields an algorithm 
remarkably similar to the predictor-corrector methods introduced in \citet{song2020score}.
For example, consider the following diffusion-like iteration,
starting from $x_T \sim \cN(0, \sigma_T)$ at $t=T$.
At timestep $t$,
\begin{enumerate}
    \item Predictor: Take one diffusion denoising step (e.g. DDIM or DDPM) w.r.t. $p_t(x \mid c)$,
    using score $ \grad_x \log p_t(x \mid c)$, to move to time $t' = t-\dt$.
    \item Corrector: Take one (or more) Langevin dynamics steps
    w.r.t. distribution $p_{t', \gamma}$,
    using score 
    $$\grad_x \log p_{t',\gamma}(x \mid c) = (1-\gamma) \grad_x \log p_{t'}(x)
+ \gamma \grad_x \log p_{t'}(x \mid c).$$
\end{enumerate}

\begin{wrapfigure}{r}{0.5\textwidth}
  \begin{center}
    \includegraphics[width=0.5\textwidth]{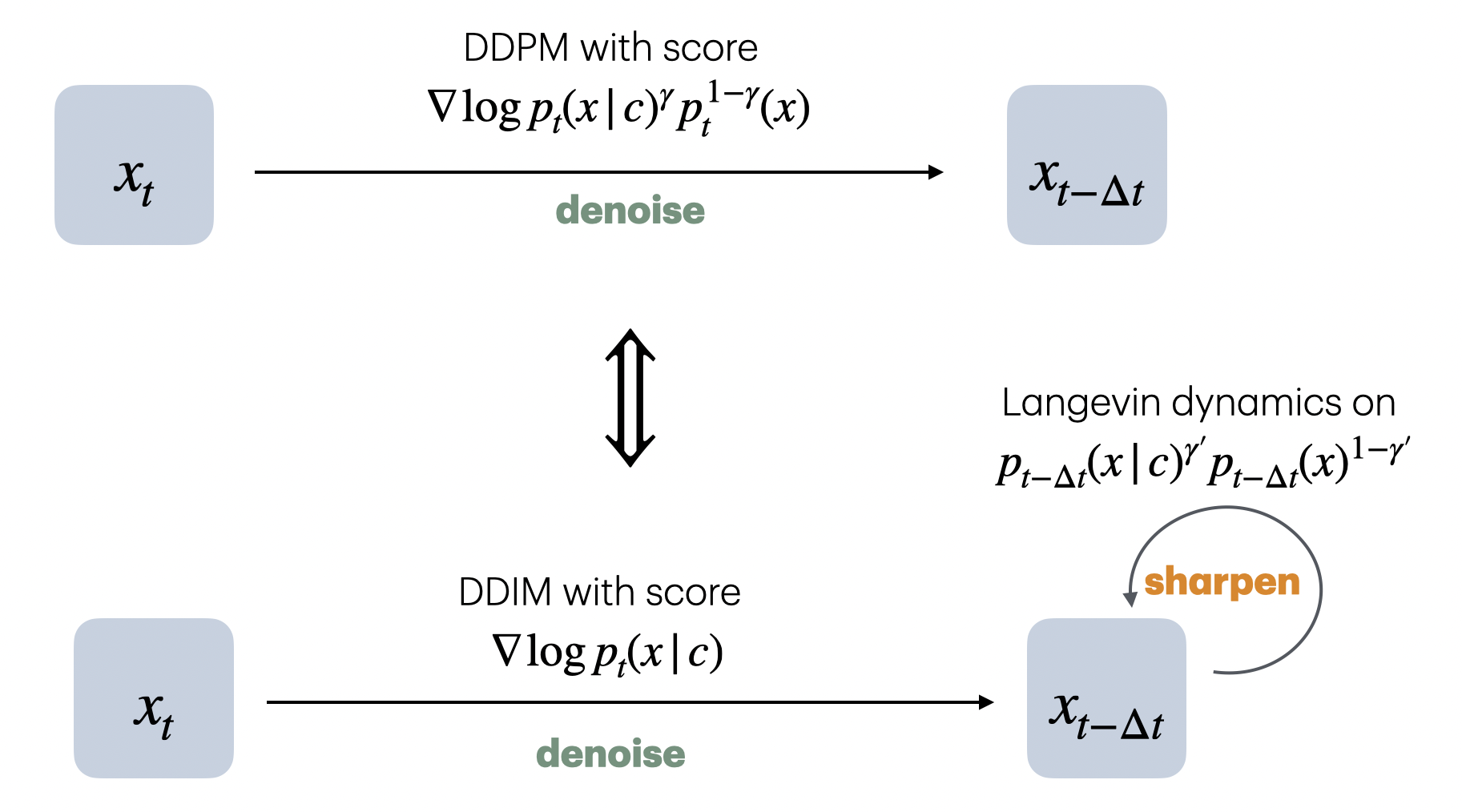}
  \end{center}
  \caption{CFG is equivalent to PCG for particular parameter choices.}
\end{wrapfigure}

It is reasonable to expect that running this iteration down to $t=0$
will produce a sample from approximately $p_\gamma(x|c)$,
since it can be thought of as annealed Langevin dynamics where the predictor
is responsible for the annealing. We name this algorithm predictor-corrector guidance (PCG). 
Notably, PCG differs from the predictor-corrector algorithms
in \citet{song2020score} because our predictor 
and corrector operate w.r.t. \emph{different} annealing distributions:
the predictor tries to anneal along the set of distributions
$\{p_{t}(x|c)\}_{t \in [0, 1]}$,
whereas the corrector anneals
along the set $\{p_{t, \gamma}(x|c)\}_{t \in [0, 1]}$.
Remarkably, it turns out that for specific choices of the denoising predictor and Langevin step size,
PCG with $K=1$ is equivalent (in the SDE limit) to CFG, but with a different $\gamma$.

\subsection{SDE limit of PCG}
\label{sec:sde-limit-pcg}

\algorithmstyle{ruled}
\begin{algorithm}[t]
\newcommand\mycommfont[1]{\normalfont\textcolor{black}{#1}}
\SetCommentSty{mycommfont}
\DontPrintSemicolon
\SetNoFillComment
\caption{$\pcg_\ddim$, theory. (see Algorithm~\ref{alg:pcg-explicit} for practical implementation.)}\label{alg:pcg}
\SetKwInput{KwData}{Constants}
\SetKwComment{Comment}{$\triangleright$\ }{}
\KwIn{Conditioning $c$, guidance weight $\gamma \geq 0$}
\KwData{$\beta_t := \beta(t)$ from \citet{song2020score}}
$x_1 \sim \cN(0, I)$ \;
\For{$(t=1-\dt;~ t \geq 0;~ t \gets t- \dt)$}{
    $s_{t+\dt} := \grad \log p_{t+\dt}(x_{t+\dt}|c)$ \;
    $x_t \gets x_{t+\dt} + \half \beta_{t} (x_{t+\dt} + s_{t+\dt})  \dt$ \Comment*{DDIM step on $p_{t+\dt}(x+\dt|c)$}
    $\eps := \beta_t \dt$ \Comment*{Langevin step size}
    \For{$k = 1,\dots K$}{
        $\eta \sim \cN(0, I_d)$ \;
        $s_{t, \gamma} := (1-\gamma)\nabla \log p_t(x_t) + \gamma \nabla \log p_t(x_t | c)$ \;
        $x_t \gets x_t +\frac{\eps}{2} s_{t, \gamma} + \sqrt{\eps} \eta$ \Comment*{Langevin dynamics on $p_{t, \gamma}(x|c)$}
    }
}
\Return $x_0$
\end{algorithm}

Consider the version of PCG defined in Algorithm \ref{alg:pcg}, which uses DDIM as predictor and a particular LD on the gamma-powered distribution as corrector. We take $K=1$, i.e. a single LD step per iteration.
Crucially, we set the LD step size such that the Langevin noise scale exactly matches
the noise scale of a (hypothetical) DDPM step at the current time
(similar to \citet{du2023reduce}).
In the limit as $\dt \to 0$, Algorithm \ref{alg:pcg} becomes the following SDE (see Appendix \ref{app:pcg_sde}):
\begin{align}
    dx &=
    \underbrace{\Delta \ddim(x, t)}_{\text{Predictor}}
    + 
    \underbrace{\Delta \LDG(x, t, \gamma)}_{\text{Corrector}}
    =: \Delta \pcg_\ddim(x, t, \gamma), \label{eq:pcg_sde} \\
    \text{where } & \Delta \ddim(x, t) = -\half \beta_{t} (x + \nabla \log p_t(x|c))  dt \notag\\
    & \Delta \LDG(x, t, \gamma) = -\half \beta_t
    \Bigl( (1-\gamma)\nabla \log p_t(x) + \gamma \nabla \log p_t(x | c) \Bigr) dt 
+ \sqrt{\beta_t} d\bar{w}. \notag
\end{align}

Above, $\Delta \ddim(x, t)$ is the \emph{differential} of the DDIM ODE \eqref{eq:ddim}, i.e. the ODE can be written as $dx = \Delta \ddim(x, t)$. 
And $\Delta \LDG(x, t, \gamma)$, where $\textsf{G}$ stands for ``guidance'', is the limit as $\dt \to 0$ of the Langevin dynamics step in PCG, which behaves like a differential of LD (see Appendix \ref{app:pcg_sde}).

We can now show that the PCG SDE \eqref{eq:pcg_sde} matches CFG, but with a different $\gamma$.  In the statement, $\Delta \cfg_\ddpm(x, t, \gamma)$ denotes the differential of the $\cfg_\ddpm$ SDE \eqref{line:ddpm-cfg}, similar to the notation above.
This result is trivial to prove using our definitions,
but the statement itself appears to be novel.
\begin{theorem}[CFG is predictor-corrector]
\label{thm:main}
In the SDE limit, CFG
is equivalent to a predictor-corrector. That is, the following differentials are equal:
{
\normalfont
\begin{align}
\Delta \cfg_\ddpm(x, t, \gamma)
&= 
\Delta \ddim(x, t) + \Delta \LDG(x, t, 2\gamma-1) 
=: \Delta \pcg_{\ddim}(x, t, 2\gamma-1)
\end{align}
}
Notably, the guidance scales of CFG
and the above Langevin dynamics are not identical.
\end{theorem}

\begin{proof}
\begin{align*}
\Delta \pcg_{\ddim}(x, t, \gamma) &= \Delta\ddim(x, t) + \Delta\LDG(x, t, \gamma) \\
&= -\half \beta_{t} (x + (1-\gamma)\nabla \log p_t(x) + (1+\gamma)\nabla \log p_t(x|c)) dt
+ \sqrt{\beta_t} d\bar{w} \\
&= -\half \beta_{t} x \dt - \beta_t \grad_x \log p_{t, \gamma'}(x|c) \dt + \sqrt{\beta_{t}} d\bar{w}, \quad \gamma' := \frac{\gamma}{2} + \frac{1}{2} \\
&= \Delta \cfg_{\ddpm}(x, t, \gamma')
\end{align*}
\end{proof}

As an aside, taking $\gamma=1$ in Theorem~\ref{thm:main}
recovers the standard fact that DDPM is equivalent, in the limit,
to DDIM interleaved with LD (e.g. \citet{karras2022elucidating}).
Because for $\gamma=1$, $\cfg_\ddpm$ is just DDPM, so Theorem~\ref{thm:main}
reduces to: $\Delta\ddpm(x, t) = \Delta\ddim(x, t) + \Delta\LDG(x, t, 1)$.
This fact, that in the non-CFG case
Langevin dynamics is equivalent to iteratively noising-then-denoising,
has been used implicitly or explicitly in a number of prior works.
For example, \citet{karras2022elucidating} use a ``churn'' operation
in their stochastic sampler, 
and \citet{lugmayr2022repaint} incorporate a conceptually similar
noise-then-denoise step in their inpainting pipeline.

\section{Discussion and Related Works}
\newcommand{\qual}{\textsf{PerceivedQuality}}

There have been many recent works toward understanding CFG.
To better situate our work, it helps 
to first discuss the overall research agenda.

\subsection{Understanding CFG: The Big Picture} 
We want to study the question of why CFG helps in practice: specifically, why it improves both image quality and prompt adherence,
compared to conditional sampling.
We can approach this question by
applying a standard generalization decomposition.
Let $p(x|c)$ be the ``ground truth'' population distribution; 
let $p^*_\gamma(x|c)$ be the distribution generated by the ideal CFG sampler, which exactly solves the CFG reverse SDE for the ground-truth scores (note that at $\gamma=1$, $p^*_1(x|c) = p(x|c)$); and let $\hat{p}_\gamma(x|c)$ denote the distribution of the real CFG sampler, with learnt scores and finite discretization.
Now, for any image distribution
$q$, let $\qual[q] \in \R$
denote a measure of perceived sample quality of this distribution to humans.
We cannot mathematically specify this notion of quality,
but we will assume it exists for analysis.
Notably, $\qual$ is \emph{not} 
a measurement of how close a distribution is
to the ground-truth $p(x|c)$ --- it is possible for a generated distribution to appear even ``higher quality'' than the ground-truth, for example.
We can now decompose:

\begin{align}
\label{eqn:gengap}
\underbrace{\qual[\hat{p}_{\gamma}]}_{\text{Real CFG}}
=
\underbrace{\qual[p^*_\gamma]}_{\text{Ideal CFG}}
-
\underbrace{\left(
\qual[p^*_\gamma]
- 
\qual[\hat{p}_\gamma]
\right)}_{\text{Generalization Gap}}.
\end{align}

Therefore, if the LHS increases with $\gamma$,
it must be because at least one of the two occurs:
\begin{enumerate}
    \item The ideal CFG sampler improves in quality with increasing $\gamma$.
    That is, CFG distorts the population distribution
    in a favorable way
    (e.g. by sharpening it, or otherwise).
    \item The generalization gap decreases with increasing $\gamma$.
    That is, CFG has a type of regularization effect,
    bringing population and empirical processes closer.
\end{enumerate}
In fact, it is likely that both occur.
The original motivation for CG and CFG
involved the first effect: CFG was intended to produce
``lower-temperature'' samples from a sharpened population distribution
\citep{dhariwal2021diffusion,ho2022classifier}.
This is particularly relevant if the model is trained on poor-quality datasets (e.g. cluttered images from the web), so we want
to use guidance to sample from a higher-quality distribution (e.g. images of an isolated subject).
On the other hand, recent studies have given evidence for the second effect.
For example, \citet{karras2024guiding} argues that unguided diffusion sampling 
produces ``outliers,'' which are avoided when using guidance --- this can be
thought of as guidance reducing the generalization gap, rather than improving the
ideal sampling distribution.
Another interpretation of the second effect is that guidance
could enforce a good inductive bias:
it ``simplifies'' the family of possible
output distributions in some sense,
and thus simplifies the learning problem,
reducing the generalization gap. Figure \ref{fig:example_4} shows a example where this occurs.
Finally, this generalization decomposition applies
to any intervention to the SDE, not just increasing guidance strength.
For example, increasing the 
Langevin steps in PCG (parameter $K$)
also shrinks the generalization gap, since
it reduces the discretization error.

\begin{figure}[p]
    \centering
    \includegraphics[width=1.0\textwidth]{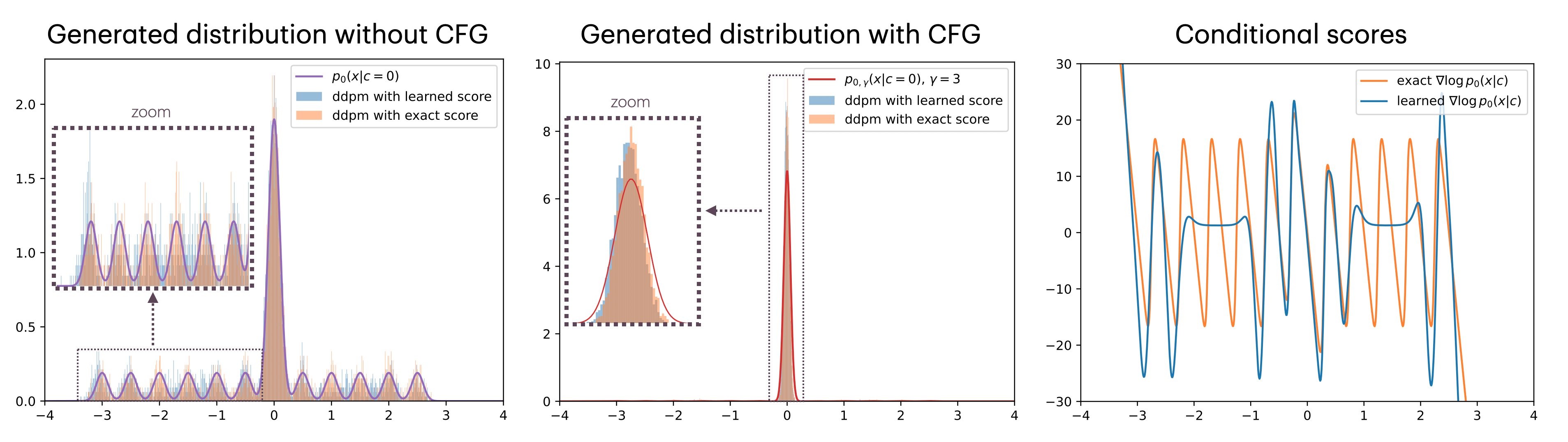}
    \caption{{\bf An example where guidance benefits generalization.} 
    Suppose that the conditional distribution for $c=0$ is a a GMM with a dominant cluster, as shown in purple, and the unconditional distribution is uniform (details in Appendix \ref{app:prob_4}). We sample with DDPM using exact scores vs. scores learned by training a small MLP with early stopping. The scores are learned more accurately near the dominant cluster. (Left) For conditional sampling (no guidance), DDPM is expected to sample from the conditional distribution (purple curve).  However, DDPM-with-learned-scores (orange) samples less accurately than DDPM-with-exact-scores (blue) away from the dominant cluster (where the learned scores are inaccurate) (note the prevalence of blue samples in low-probability regions). (Center) With guidance $\gamma=3$, $p_{0,\gamma}(x|c=0)$ (red) and both samplers concentrate around the dominant cluster (where the learned scores are accurate), reducing the generalization gap between the learned and exact models. (Right) Exact vs. learned condition scores $\grad \log p(x|c=0)$. }
    \label{fig:example_4}
\end{figure}

In this framework,
our work makes progress towards understanding 
both terms on the RHS of Equation~\ref{eqn:gengap}, in different ways.
For the first term,
we identify structural properties of ideal CFG,
by showing that $p^*_\gamma$ can be equivalently generated by
a standard technique (an annealed Langevin dynamics).
For the second term,
the PCG framework
highlights the ways in which errors in the 
learned score can contribute to a
generalization gap, during both the denoising step
and the LD step (the latter would move toward an inaccurate steady-state distribution).

\subsection{Open Questions and Limitations}
In addition to the above, there are a number of other questions left open by our work.
First, we study only the stochastic variant of CFG (i.e. $\cfg_\ddpm$),
and it is not clear how to adapt our analysis to the more commonly used deterministic variant
($\cfg_\ddim$).
This is subtle because the two CFG variants
can behave very differently in theory,
but appear to behave similarly in practice.
It is thus open to identify plausible theoretical conditions
which explain this similarity\footnote{
Curiously, $\cfg_\ddim$ is the correct probability-flow ODE
for $\cfg_\ddpm$ if and only if the true intermediate 
distribution at time $t$ is $p_{t, \gamma}$. 
However we know this is not the true
distribution in general, from Section~\ref{sec:misconceptions}.
}; we give a suggestive experiment in Figure~\ref{fig:counterexample_extras}.
More broadly, it is open to find explicit characterizations of CFG's output distribution,
in terms of the original 
$p(x)$ and $p(x|c)$
--- although it is
possible tractable expressions do not exist.

Finally, we presented PCG primarily as a tool to understand CFG,
not as a practical algorithm in itself.
Nevertheless, the PCG framework outlines a broad family of guided samplers,
which may be promising to explore in practice.
For example, the predictor can be any diffusion denoiser,
including CFG itself.
The corrector can operate on any 
distribution with a known score, including compositional distributions as in \citet{du2023reduce}, or any other distribution that might help sharpen or otherwise improve on the conditional distribution. Finally, the number of Langevin steps could be adapted to the timestep, similar to \citet{kynkaanniemi2024applying}, or alternative samplers could be considered \citep{du2023reduce, neal2012mcmc, ma2015complete}.

\subsection{Stable Diffusion Examples}
\label{sec:experiments}

\begin{figure}[p]
    \centering
    \includegraphics[width=1.0\textwidth]{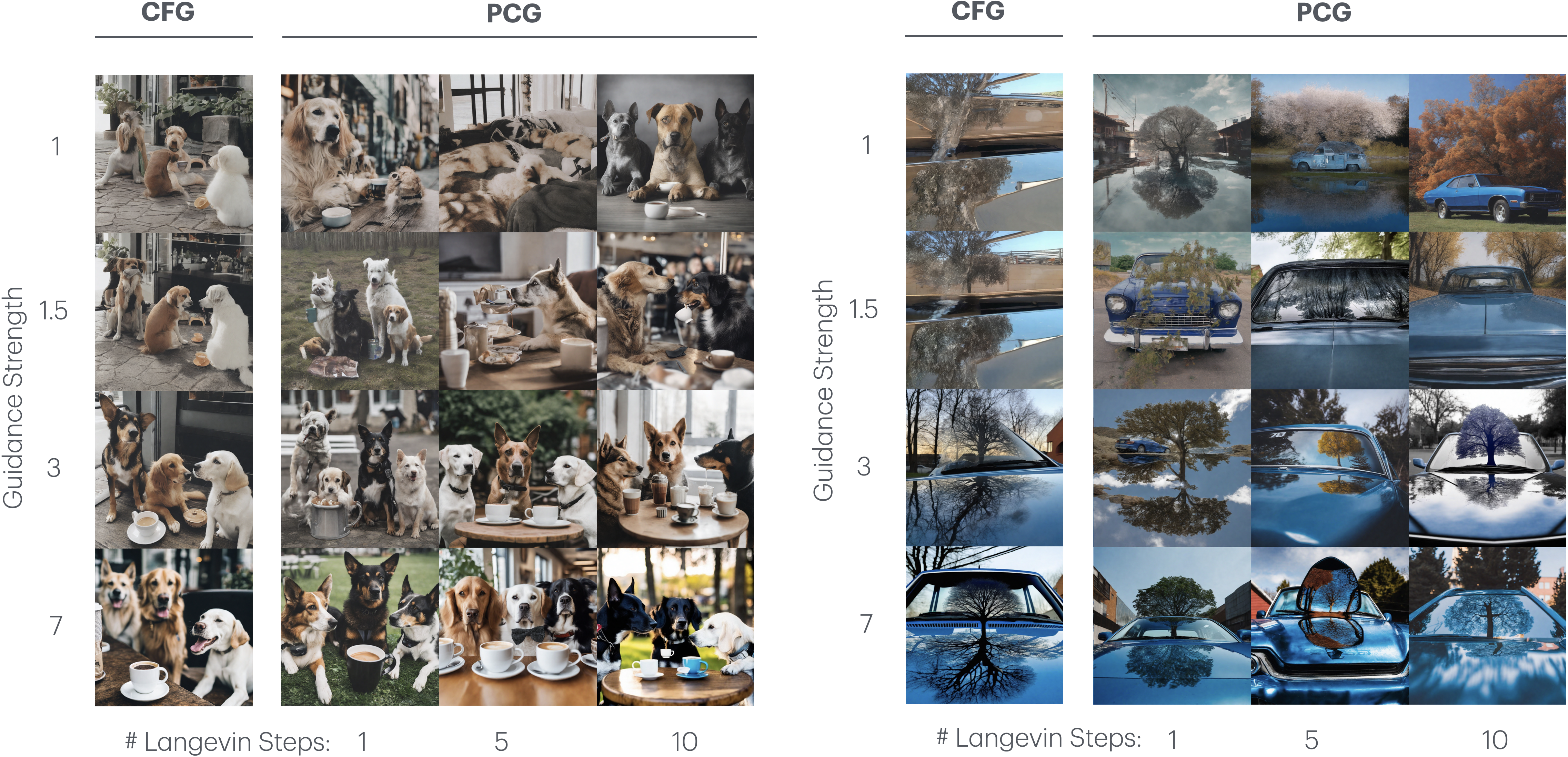}
    \caption{{\bf Effect of Guidance and Correction.}
    Each grid shows SDXL samples using $\pcg_\ddim$,
    as the guidance strength $\gamma$ and Langevin iterations $K$ are varied.
    Left: ``photograph of a dog drinking coffee with his friends''. Right: ``a tree reflected in the hood of a blue car''.
    (Zoom in to view).
    }
    \label{fig:pcg_grid}
\end{figure}

We include several examples 
running predictor-corrector guidance on 
Stable Diffusion XL \citep{podell2023sdxl}.
These serve primarily to sanity-check our theory,
not as a suggestion for practice.
For all experiments, we use $\pcg_\ddim$
as implemented explicitly
in Algorithm~\ref{alg:pcg-explicit}\footnote{Note that Algorithm~\ref{alg:pcg} and~\ref{alg:pcg-explicit}
have slightly different DDIM steps, but this just corresponds to two different discretizations of the same process.
Algorithm~\ref{alg:pcg} uses the first-order Euler–Maruyama
discretization known as ``reverse SDE'' \citep{song2020score},
which is convenient for our mathematical analysis.
Algorithm~\ref{alg:pcg-explicit} uses the
original DDIM discretization \citep{song2021denoising},
equivalent to a more sophisticated integrator \citep{lu2022dpm},
which is more common in practice.}.
Note that PCG offers a more flexible design space than standard CFG;
e.g. we can run multiple corrector steps for each denoising step
to improve the quality of samples (controlled by parameter $K$
in Algorithm~\ref{alg:pcg-explicit}).

\paragraph{CFG vs. PCG.}
Figure~\ref{fig:cfg_vs_pcg} illustrates the equivalence of Theorem~\ref{thm:main}: we compare $\cfg_\ddpm$ with guidance $\gamma$
to $\pcg_\ddim$ with exponent $\gamma' := (2\gamma-1)$.
We run $\cfg_\ddpm$ with 200 denoising steps,
and $\pcg_\ddim$ with 100 denoising steps
and $K=1$ Langevin corrector step
per denoising step.
Corresponding samples appear to have
qualitatively similar guidance strengths, consistent with our theory.

\paragraph{Effects of Guidance and Corrector.}
In Figure~\ref{fig:pcg_grid} we show samples from
$\pcg_\ddim$, varying the guidance strength and Langevin iterations
(i.e. parameters $\gamma$ and $K$ respectively
in Algorithm~\ref{alg:pcg-explicit}).
We also include standard $\cfg_\ddim$ samples for comparison.
All samples used 1000 denoising steps for the base predictor.
Overall, we observed that increasing Langevin steps
tends to improve the overall image quality,
while increasing guidance strength tends to 
improve prompt adherence.
In particular, sufficiently many Langevin steps
can sometimes yield high-quality conditional samples,
even \emph{without} any guidance ($\gamma=1$);
see Figure~\ref{fig:panda} in the Appendix for another
such example.
This is consistent with the observations of 
\citet{song2020score} on unguided predictor-corrector methods.
It is also related to the findings of \citet{du2023reduce} on MCMC methods:
\citet{du2023reduce}  
similarly use an annealed Langevin dynamics
with reverse-diffusion annealing,
although they focus on general compositions of distributions
rather than the specific gamma-powered distribution of CFG.

Notice that in Figure~\ref{fig:pcg_grid},
increasing the number of Langevin steps
appears to also increase the ``effective'' guidance strength.
This is because 
the dynamics does not fully mix:
one Langevin step ($K=1$) does not
suffice to fully converge the
intermediate distributions
to $p_{t, \gamma}$.

\section{Conclusion}

In this paper, we have shown that while CFG is not a diffusion sampler on the gamma-powered data distribution $p_0(x)^{1-\gamma} p_0(x|c)^\gamma$, it can be understood as a particular kind of predictor-corrector, where the predictor is a DDIM denoiser, and the corrector at each step $t$ is one step of Langevin dynamics on the gamma-powered noisy distribution $p_t(x)^{1-{\gamma'}} p_t(x|c)^{\gamma'}$, with $\gamma' = (2\gamma-1)$. Although \citet{song2020score}'s Predictor-Corrector algorithm has not been widely adopted in practice, perhaps due to its computation expense relative to samplers like DPM++ \citep{lu2022dpmplusplus}, it turns out to provide a lens to understand the unreasonable practical success of CFG.
On a practical note, PCG encompasses a rich design space of possible predictors and correctors for future exploration, that may help improve the prompt-alignment, diversity, and quality of diffusion generation.

{\bf Acknowledgements.}
We thank
David Berthelot,
James Thornton,
Jason Ramapuram,
Josh Susskind,
Miguel Angel Bautista Martin,
Jiatao Gu,
Zijing Ou
and
Rob Brekelmans
for helpful discussions and feedback throughout this work.
\clearpage
\newpage

\bibliographystyle{plainnat}
\bibliography{refs}

\appendix
\section{1D Gaussian Counterexamples}

\begin{figure}[h]
    \centering
    \includegraphics[width=0.49\textwidth]{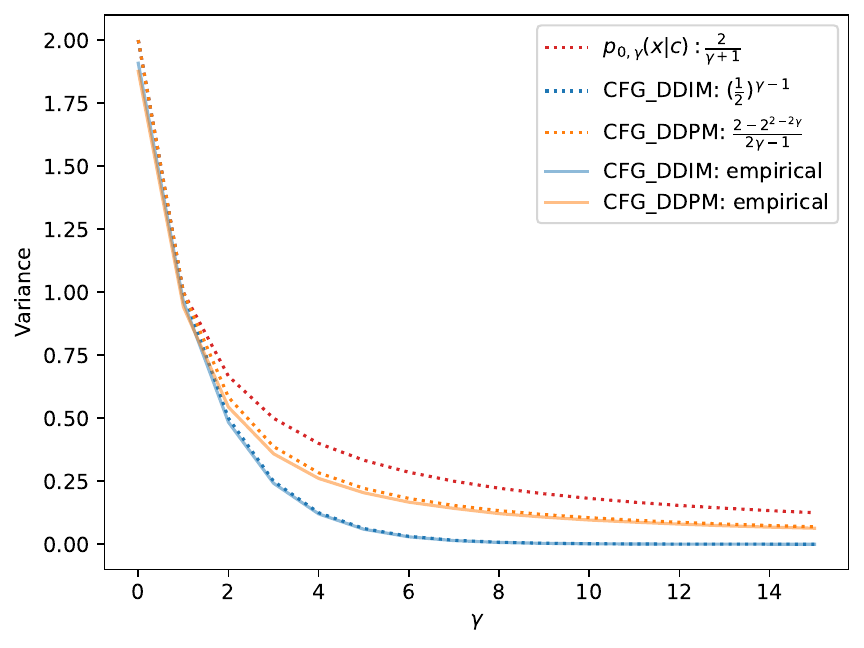}
    \includegraphics[width=0.49\textwidth]{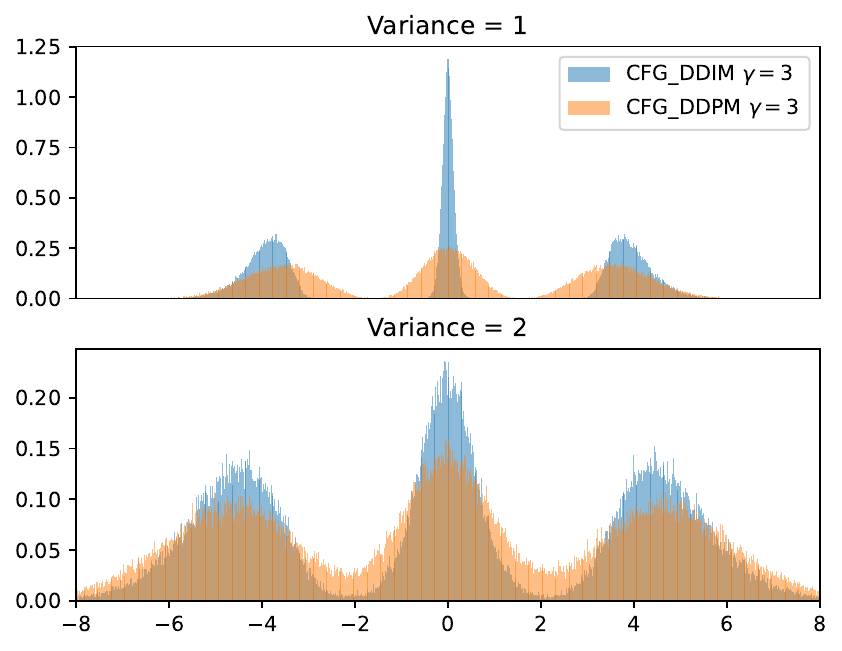}
    \caption{(Left) For Counterexample 1 (section \ref{sec:counterex1}), we plot the empirical and theoretical variance of the gamma-powered, $\cfg_\ddim$, and $\cfg_\ddpm$ distributions, over a range of values of $\gamma$. The theoretical predictions are given by equations \eqref{eq:prob_1_ddim} and \eqref{eq:prob_1_ddpm}, and the empirical distributions are sampled using an exact denoiser. This verifies the theoretical predictions and illustrates the decreasing variance from $p_{0,\gamma}$ to $\cfg_\ddpm$ to $\cfg_\ddim$.
    (Right) For counterexample 3 (section \ref{app:prob_3} with different choices of variance ($\sigma=1$ and $\sigma=2$), we compare $\cfg_\ddim$ and $\cfg_\ddpm$. Increasing the variance makes the two CFG samplers more similar. Also note that the $\cfg_\ddim$ distribution is symmetric around the center cluster, but asymmetric around the side clusters. This experiment suggests that multiple clusters and greater overlap between classes can help symmetrize and reduce the difference between $\cfg_\ddim$ and $\cfg_\ddpm$}
    \label{fig:counterexample_extras}
\end{figure}

\subsection{Counterexample 1 Detail}
\label{app:prob_1}
Counterexample 1 (equation \ref{eq:1d_gauss_problem_1}) has
\begin{align*}
p_0(x) &\sim \cN(0,2) \\
p_0(x|c=0) &\sim \cN(0,1).
\end{align*}
The $\gamma$-powered distribution is
\begin{align*}
    p_{0, \gamma}(x|c=0) 
    &= p_{0}(x|c)^{\gamma} p_{c=0}(x)^{1 - \gamma} \notag\\
    &\propto e^{-\frac{\gamma x^2}{2}} e^{-\frac{(1-\gamma) x^2}{4}} 
    = e^{-\frac{(\gamma +1) x^2}{4}} \\
    &\sim \cN(0, \frac{2}{\gamma+1}).
\end{align*}
We consider the simple variance-exploding diffusion defined by the SDE
$$dx = \sqrt{t} dw.$$
The DDIM sampler is a discretization of the reverse ODE
$$ \frac{dx}{dt} = -\half \nabla_x \log p_t(x), $$
and the DDPM sampler is a discretization of the reverse SDE 
$$dx = -\nabla_x \log p_{t, \gamma}(x) dt + d\bar w.$$
For $\cfg_\ddim$ or $\cfg_\ddpm$, we replace the score with CFG score $\nabla_x \log p_{t, \gamma} (x)$.

During training we run the forward process until some time $t=T$, at which point we assume it is fully-noised, so that approximately
$$p_T(x|c=0) \sim \cN(0, T)$$
(in this case the exact distribution $p_T(x|c=0) \sim \cN(0, T+1)$ so we need to choose $T \gg 1$ to ensure sufficient terminal noise). At inference time we choose an initial sample $x_T \sim \cN(0, T)$ and run $\cfg_\ddim$ from $t = T \to 0$ to obtain a final sample $x_0$.

\paragraph{$\cfg_\ddim$}
For Counterexample 1, the $\cfg_\ddim$ ODE has a closed-form solution (derivation in section \ref{app:ode_solutions}):
\begin{align*}
    \cfg_\ddim: \quad 
    \frac{dx}{dt} &= -\half \nabla_x \log p_{t, \gamma}(x) \\
    &= x_t \left( \frac{\gamma}{2 (1 + t)} + \frac{(1-\gamma)}{2 (2 + t)} \right) \\
    \implies x_t &= x_T \sqrt{\frac{(t+1)^{\gamma} (t+2)^{1-\gamma} }{(T+1)^{\gamma} (T+2)^{1-\gamma}}}.
\end{align*}
That is, for a particular initial sample $x_T$, $\cfg_\ddim$ produces the sample $x_t$ at time $t$. Evaluating at $t=0$ and taking the limit as $T \to \infty$ yields the ideal denoised $x_0$ sampled by $\cfg_\ddim$ given an initial sample $x_T$:
\begin{align*}
    \hat x_0^\text{$\cfg_\ddim$} (x_T) &= x_T \sqrt{\frac{2^{1-\gamma} }{(T+1)^{\gamma} (T+2)^{1-\gamma}}} \\
    &\to x_T \sqrt{\frac{2^{1-\gamma} }{T}} \quad \text{as } T \to \infty.
\end{align*}
To get the denoised distribution obtained by reverse-sampling with $\cfg_\ddim$, we need to average over the distribution of $x_T$:
\begin{align*}
    \E_{x_T \sim \cN(0,T)}[\hat x_0^\text{$\cfg_\ddim$}(x_T)]
    = \cN(0, T \frac{2^{1-\gamma} }{T})
    = \cN \left( 0, 2^{1-\gamma} \right).
\end{align*}
which is equation \ref{eq:prob_1_ddim} in the main text.

\paragraph{$\cfg_\ddpm$}
$\cfg_\ddpm$ also has a closed-form solution (derived in section \ref{app:ode_solutions}):
\begin{align*}
    dx &= -\nabla_x \log p_{t, \gamma}(x) dt + d\bar w \\
    &= x \left( \frac{\gamma}{(1 + t)} + \frac{(1-\gamma)}{(2 + t)} \right)dt + d\bar w \\
    \implies x(t)
    &= x_T \frac{(1 + t)^{\gamma}(2+t)^{1-\gamma}}{(1 + T)^{\gamma}(2+T)^{1-\gamma}} 
    + (1 + t)^{\gamma}(2+t)^{1-\gamma} \sqrt{\frac{1}{2\gamma-1}} \sqrt{ \left( \frac{t+1}{t+2} \right)^{1 - 2\gamma} - \left( \frac{T+1}{T+2} \right)^{1 - 2\gamma} } \xi.
\end{align*}
Similar to the $\cfg_\ddim$ argument, we can obtain the final denoised distribution as follows:
\begin{align*}
    \hat x_0^\text{$\cfg_\ddpm$} (x_T) &= x_T \frac{2^{1-\gamma}}{(1 + T)^{\gamma}(2+T)^{1-\gamma}} 
    + 2^{1-\gamma} \sqrt{\frac{1}{2\gamma-1}} \sqrt{ 2^{2\gamma-1} - \left(\frac{T+1}{T+2} \right)^{1 - 2\gamma}} \xi \\ 
    &\to x_T \frac{2^{1-\gamma}}{T} 
    + \sqrt{\frac{2 - 2^{2-2\gamma}}{2\gamma-1}} \xi \quad \text{as } T \to \infty \\ 
    \implies \E_{x_T \sim \cN(0,T)}[\hat x_0^\text{$\cfg_\ddpm$}(x_T)]
    &= \cN \left( 0, T \left(\frac{2^{1-\gamma}}{T} \right)^2 + \frac{2 - 2^{2-2\gamma}}{2\gamma-1}\right) \\
    &\to \cN \left( 0, \frac{2 - 2^{2-2\gamma}}{2\gamma-1}\right),
\end{align*}
which is equation \ref{eq:prob_1_ddpm} in the main text, and for $\gamma \gg 1$ becomes approximately
$$\E_{x_T \sim \cN(0,T)}[\hat x_0^\text{$\cfg_\ddpm$}(x_T)] \approx \cN \left( 0, \frac{2}{2\gamma-1} \right).$$

In Figure \ref{fig:counterexample_extras}, we confirm results (\ref{eq:prob_1_ddpm}, \ref{eq:prob_1_ddim}) empirically.

\subsection{Counterexample 2}
\label{app:prob_2}
Counterexample 2 \eqref{eq:1d_gauss_problem_1} is a Gaussian mixture with equal weights and variances.
\begin{align*}
    c &\in \{0, 1\}, \quad p(c = 0) = \half \notag \\ 
    p_0(x_0 | c) &\sim \cN(\mu^{(c)}, 1), \quad \mu^{(0)} = -\mu, \quad \mu^{(1)} = \mu \notag\\
    p_0(x_0) &\sim \half p_0(x_0 | c=0) + \half p_0(x_0 | c=1).
\end{align*}

We noted in the main text that if $\mu$ is sufficiently large enough that the clusters are approximately disjoint, and $\gamma \ge 1$, then $p_{0, \gamma} (x|c) \approx p_{0} (x|c)$. To see this note that
\begin{align*}
    p_0(x_0) &\approx \half p_0(x_0 | 0) \mathbbm{1}_{x > 0} + \half p_0(x_0 | 1) \mathbbm{1}_{x > 0} \\
    p_{0, \gamma} (x|c) &\propto p_0(x|c)^\gamma p_0(x)^{1-\gamma} \\
    &= p_0(x) \left( \frac{p_0(x|c)}{p_0(x)} \right)^\gamma  \\
    &\propto p_0(x) \left( \mathbbm{1}_{\text{sign}(x) = \mu^{(c)}} \right)^\gamma  \\
    &\approx p_0(x|c) \quad \text{for } \gamma \ge 1.
\end{align*}
However, $p_{t, \gamma}(x|c) \neq p_t(x|c)$ since the noisy distributions do overlap/interact.

We don't have complete closed-form solutions for this problem like we did for Counterexample 1. We have the solution for conditional DDIM for the basic VE process $dx = dw$ (using the results from the previous section):
\begin{align*}
    \text{DDIM on $p_t(x|c)$: } \frac{dx}{dt} &= -\half \nabla_x \log p_t(x|c) \\
    &= -\frac{1}{2(1 + t)} (\mu^{(c)} - x_t) \\
    \implies x(t) &= \mu^{(c)} + (x_T - \mu^{(c)}) \sqrt{\frac{1 + t}{1 + T}},
\end{align*}
but otherwise have to rely on empirical results. We do however have access to the ideal conditional and unconditional denoisers via the scores (Appendix \ref{app:exact_denoiser}):
\begin{align*}
    \nabla_x \log p_t(x|c) &= -\frac{1}{2(1 + t)} (\mu^{(c)} - x_t) \\
    \nabla_x \log p_t(x) &= \frac{\nabla_x p_t(x)}{p_t(x)} 
    = \frac{\half \sum_{c=0,1} \nabla_x p_t(x | c)}{p_t(x)}.
\end{align*}

\subsection{Counterexample 3}
\label{app:prob_3}
We consider a 3-cluster problem to investigate why $\cfg_\ddim$ and $\cfg_\ddpm$ often appear similar in practice despite being different in theory. Counterexample 3 \eqref{eq:1d_gauss_problem_1} is a Gaussian mixture with equal weights and variances. We vary the variance to investigate its effect on CFG.
\begin{align*}
    c &\in \{0, 1, 2\}, \quad p(c) = \frac{1}{3} \quad \forall c \notag \\ 
    p_0(x_0 | c) &\sim \cN(\mu^{(c)}, \sigma), \quad \mu^{(0)} = -3, \quad \mu^{(1)} = 0, \quad \mu^{(2)} = 3 \notag\\
    p_0(x_0) &\sim \frac{1}{3} p_0(x_0 | c=0) + \frac{1}{3} p_0(x_0 | c=1) + \frac{1}{3} p_0(x_0 | c=2).
\end{align*}
We run $\cfg_\ddim$ and $\cfg_\ddpm$ with $\gamma=3$, for $\sigma = 1$ and $\sigma=2$. Results are shown in Figure \ref{fig:counterexample_extras}.

\subsection{Generalization Example 4}
\label{app:prob_4}

We consider a multi-cluster problem to explore the impact of guidance on generalization:
\begin{align}
p_0(x) &\sim \cN(0, 10) \notag \\
p_0(x|c=0) &\sim \sum_i w_i \cN(\mu_i, \sigma) \label{eq:prob_4} \\
\mu &= (-3, -2.5, -2, -1.5, -1, -0.5,  0,  0.5, 1,  1.5,  2,  2.5) \notag \\
w_i &= 0.0476 \quad \forall i \neq 6; \quad w_6 = 0.476 \notag\\
\sigma &= 0.1 \notag
\end{align}
Note that the unconditional distribution is wide enough to be essentially uniform within the numerical support of the conditional distribution. The conditional distribution is a GMM with evenly spaced clusters of equal variance, and all equal weights, except for a ``dominant" cluster in the middle with higher weight. The results are shown in Figure \ref{fig:example_4}.

\subsection{Closed-form ODE/SDE solutions}
\label{app:ode_solutions}

First, we want to solve equations of the general form $\frac{dx}{dt} = -a(t) x + b(t)$, which will encompass the ODEs and SDEs of interest to us. All we need for the ODEs is the special $b(t) = a(t)c$, which is easier.

The main results are
\begin{align}
    \frac{dx}{dt} &= a(t)(c - x) \notag \\
    \implies x(t) &= c + (x_T - c)e^{A(T)-A(t)} \label{eq:special_ode_final} \\
    \text{where } A(t) &= \int a(t) dt \notag 
\end{align}
and
\begin{align}
    \frac{dx}{dt} &= -a(t) x + b(t) \notag\\
    \implies x(t) &= e^{-A(t)} (B(t) - B(T)) +  x_T e^{A(T)-A(t)}  \label{eq:general_ode_final} \\
    \text{where } A(t) &= \int a(t) dt, \quad B(t) = \int e^{A(t)} b(t) dt. \notag 
\end{align}

First let's consider the special case $b(t) = a(t)c$, which is easier. We can solve it (formally) by separable equations:
\begin{align}
    \frac{dx}{dt} &= a(t)(c - x) \notag \\
    \implies \int \frac{1}{c - x} dx &= \int a(t) dt = A(t) \notag \\
    \implies -\log(c-x) &= A(t) + C \notag \\
    \implies c-x &= e^{-A(t) - C} \notag \\
    \implies x(t) &= c + C e^{-A(t)}. \label{eq:special_ode_formal}
\end{align}
Next we need to apply initial conditions to get the right constants.
Remembering that we are actually sampling backward in time from initialization $x_T$, we can solve for the constant $C$ as follows, to obtain result \eqref{eq:special_ode_final}:
\begin{align*}
    x_T &= c + C e^{-A(T)} \\
    \implies C &= e^{A(T)}(x_T - c) \\
    \implies x(t) &= c + (x_T - c)e^{A(T)-A(t)}.
\end{align*}

We will apply this result to $\cfg_\ddim$ shortly, but for now we note that for a VE diffusion $dx = \sqrt{t} dw$ on a Gaussian data distribution $p_0(x) \sim \cN(\mu, \sigma)$ the above result implies the exact DDIM dynamics:
\begin{align*}
    p_t(x) \sim \cN(\mu, \sigma^2 + t) \\
    \text{DDIM on $p_t(x)$: } \frac{dx}{dt} &= -\half \nabla_x \log p_t(x) \\
    &= -\frac{1}{2(\sigma^2 + t)} (\mu - x) \\
    A(t) &= -\half \log(\sigma^2 + t) \\
    \implies x_t &= \mu + (x_T - \mu) e^{A(T) - A(t)} \\
    &= \mu + (x_T - \mu) \sqrt{\frac{\sigma^2 + t}{\sigma^2 + T}}.
\end{align*}
(which makes sense since $x_{t=T} = x_T$ and $\frac{\sqrt{\sigma^2}}{\sqrt{\sigma^2 + T}} \approx 0 \implies x_{t=0} \approx \mu$).
    
Now let's return to the general problem with arbitrary $b(t)$ (we need this for the SDEs). We can use an integrating factor to get a formal solution:
\begin{align}
    \frac{dx}{dt} &= -a(t) x + b(t) \notag \\
    \text{Integrating factor: } & e^{A(t)}, \quad A(t) = \int a(t) dt \notag \\
    \frac{d}{dt} (x(t) e^{A(t)}) &= \left(x'(t) + a(t) x(t) \right)e^{A(t)} \notag \\
    &= b(t) e^{A(t)} \notag \\
    \implies e^{A(t)} x(t) &= \int e^{A(t)} b(t) dt + C\notag  \\
    \implies x(t) &= e^{-A(t)} \int e^{A(t)} b(t) dt + C e^{-A(t)}. \label{eq:general_ode_formal}
\end{align}
Note that if $b(t) = a(t)c$ this reduces to \eqref{eq:special_ode_formal}:
\begin{align*}
    \int e^{-A(t)} e^{A(t)} b(t) dt &= c e^{-A(t)} \int a(t) e^{A(t)} dt = c \\
    \implies x(t) &= c + C e^{-A(t)}.
\end{align*}

Again, we need to apply boundary conditions to get the constant, and remember that we are actually sampling backward in time from initialization $x_T$ to obtain result \eqref{eq:general_ode_final}:
\begin{align*}
    \frac{dx}{dt} &= -a(t) x + b(t) \\
    x_T &= e^{-A(T)} B(T) + C e^{-A(T)}, \quad B(t) := \int e^{A(t)} b(t) dt \\
    \implies C &= e^{A(T)} x_T - B(T) \\
    \implies x(t) &= e^{-A(t)} B(t) + (e^{A(T)} x_T - B(T)) e^{-A(t)} \\
    &= e^{-A(t)} (B(t) - B(T)) +  x_T e^{A(T)-A(t)}.
\end{align*}

Note that for $b(t) = a(t)c$ this reduces \eqref{eq:special_ode_final}:
\begin{align*}
    b(t) &= a(t)c \implies B(t) = ce^{A(t)}  \\
    \implies x(t) &= -c e^{-A(t)} (e^{A(t)} - e^{A(T)}) +  x_T e^{A(T)-A(t)} \\
    &= c + (x_T - c) e^{A(T)-A(t)}.
\end{align*}

\paragraph{Counterexample 1 solutions}

To solve the $\cfg_\ddim$ ODE for Counterexample 1 (Equation \ref{eq:1d_gauss_problem_1}) we apply result \eqref{eq:special_ode_final}:
\begin{align*}
    \frac{dx}{dt} &= a(t)(c - x) \implies x(t) = c + (x_T - c)e^{A(T) - A(t)} \\
    a(t) &= -\frac{\gamma}{2 (1 + t)} - \frac{(1-\gamma)}{2 (2 + t)}, \quad c = 0 \\
    A(t) &= -\half \int \frac{\gamma}{(1 + t)} + \frac{(1-\gamma) }{(2 + t)} dt \\
    &= -\half (\gamma \log(t+1) + (\gamma-1)\log(t+2)) \\
    \implies x_t &= x_T \sqrt{\frac{(t+1)^{\gamma} (t+2)^{1-\gamma} }{(T+1)^{\gamma} (T+2)^{1-\gamma}}}.
\end{align*}

To solve the $\cfg_\ddpm$ SDE for Counterexample 1 (Equation \ref{eq:1d_gauss_problem_1}), we first apply \eqref{eq:general_ode_final} to the SDE with $b(t) = -\xi(t)$:
\begin{align*}
    \frac{dx}{dt} &= -a(t)x - \xi(t), \quad \< \xi(t) \> = 0, \quad \< \xi(t), \xi(t') \> = \delta(t-t') \\
    \implies x(t) &= x_T e^{A(T)-A(t)} + e^{-A(t)} (B(t) - B(T)), \quad A(t) = \int a(t) dt, \quad B(t) = -\int e^{A(t)} \xi(t) dt \\
    &= x_T e^{A(T)-A(t)} + e^{-A(t)} \sqrt{\int_t^T e^{2A(t)} dt} \xi.
\end{align*}
Now, plugging in the DDPM drift term we find that
\begin{align*}
    a(t) &= -\frac{\gamma}{(1 + t)} - \frac{(1-\gamma)}{(2 + t)}\\
    A(t) &= -\gamma \log(1 + t) - (1-\gamma) \log(2+t) \\
    e^{A(t)} &= (1 + t)^{-\gamma}(2+t)^{-1 + \gamma}\\
    \int e^{2A(t)} dt &= \int (1 + t)^{-2\gamma}(2+t)^{-2 + 2\gamma} dt \\
    &= -\frac{1}{2\gamma-1} \left( \frac{t+1}{t+2} \right)^{1 - 2\gamma} \\
    x(t) &= x_T e^{A(T)-A(t)} + e^{-A(t)} \sqrt{\int_t^T e^{2A(t)} dt} \xi \\
    &= x_T \frac{(1 + t)^{\gamma}(2+t)^{1-\gamma}}{(1 + T)^{\gamma}(2+T)^{1-\gamma}} 
    + (1 + t)^{\gamma}(2+t)^{1-\gamma} \sqrt{\frac{1}{2\gamma-1}} \sqrt{ \left( \frac{t+1}{t+2} \right)^{1 - 2\gamma} - \left( \frac{T+1}{T+2} \right)^{1 - 2\gamma} } \xi. \\ 
\end{align*}

\subsection{Exact Denoiser for GMM}
\label{app:exact_denoiser}
For the experiments in Figure~\ref{fig:ddpm_vs_ddim_counterexamples}, we used an exact denoiser, for which we require exact conditional and unconditional scores. Exact scores are available for any GMM as follows. This is well-known (e.g. \citet{karras2024guiding}) but repeated here for convenience.

\begin{align*}
    p(x) = \sum w_i \phi(x; \mu_i, \sigma_i), & \quad
    \text{where} \quad \phi(x; \mu, \sigma^2) := \frac{1}{\sqrt{2 \pi} \sigma} e^{-\frac{(x - \mu)^2}{2\sigma^2}} \\
    \implies \grad \log p(x) &= \frac{\grad p(x)}{p(x)} \\
    &= \frac{\sum w_i \grad \phi(\mu_i, \sigma_i)}{\sum w_i \phi(\mu_i, \sigma_i)} \\
    &= - \frac{\sum w_i \left(\frac{x - \mu_i}{\sigma_i^2} \right) \phi(x; \mu_i, \sigma_i^2)}{\sum w_i \phi(\mu_i, \sigma_i)}.
\end{align*}

\section{PCG SDE}
\label{app:pcg_sde}

We want to show that the SDE limit of Algorithm \ref{alg:pcg} with $K=1$ is
\begin{align*}
    dx &= \Delta \ddim(x, t) + \Delta \LDG(x, t, \gamma).
\end{align*}
To see this, note that
a single iteration of Algorithm \ref{alg:pcg} with $K=1$ expands to
\begin{align*}
    x_t &=  x_{t+\dt} \underbrace{-\half \beta_{t} (x_{t+\dt} - \nabla \log p_{t+\dt}(x_{t+\dt}|c))\dt}_{\text{DDIM step on } p_{t+\dt}(x+\dt|c)}
    + \underbrace{\frac{\beta_t\Delta{t}}{2} \grad \log p_{t,\gamma}(x_t|c) + \sqrt{\beta_t\Delta{t}} \cN(0, I_d)}_{\text{Langevin dynamics on } p_{t, \gamma}(x|c)} \\
    \implies dx &= \lim_{\dt \to 0} x_t - x_{t+\dt} 
    = \underbrace{-\half \beta_{t} (x_{t} - \nabla \log p_{t}(x_{t}|c))dt}_{\Delta \ddim(x, t)}
    + \underbrace{\half \beta_t \grad \log p_{t,\gamma}(x_t|c) dt + \sqrt{\beta_t} d\bar w}_{\Delta \LDG(x, t, \gamma)}.
\end{align*}
This concludes the proof.

A subtle point in the argument above is that $\Delta \LDG(x, t, \gamma)$ represents the result of the Langevin step in the PCG corrector update, rather than the differential of an SDE. In Algorithm \ref{alg:pcg}, $t$ remains constant during the LD iteration, and so the SDE corresponding to the LD iteration is
\begin{align}
dx = \half \beta_t \grad \log p_{t,\gamma}(x_t|c) ds + \sqrt{\beta_t} d\bar w,
\label{eq:ldg_ds}
\end{align}
where $s$ is an LD time-axis that is distinct from the denoising time $t$, which is fixed during the LD iteration. Thus $\Delta \LDG(x, t, \gamma)$ is not the differential of \eqref{eq:ldg_ds} (the difference is $dt$ vs $ds)$. However, when we take an LD step of length $dt$ as required for the PCG corrector, the result is
\begin{align*}
   \int_0^{dt} -\frac{\beta_t}{2}  \grad \log p_{t, \gamma}ds + \sqrt{\beta_t} d\bar{w} = -\frac{\beta_t}{2}  \grad \log p_{t, \gamma}dt + \sqrt{\beta_t} d \bar{w} = \Delta \LDG(x, t, \gamma),
\end{align*}
so $\Delta \LDG(x, t, \gamma)$ represents the result of the PCG corrector update in the limit as $\dt \to 0$.

\section{Additional Samples}

\begin{figure}
    \centering
    \includegraphics[width=1.0\textwidth]{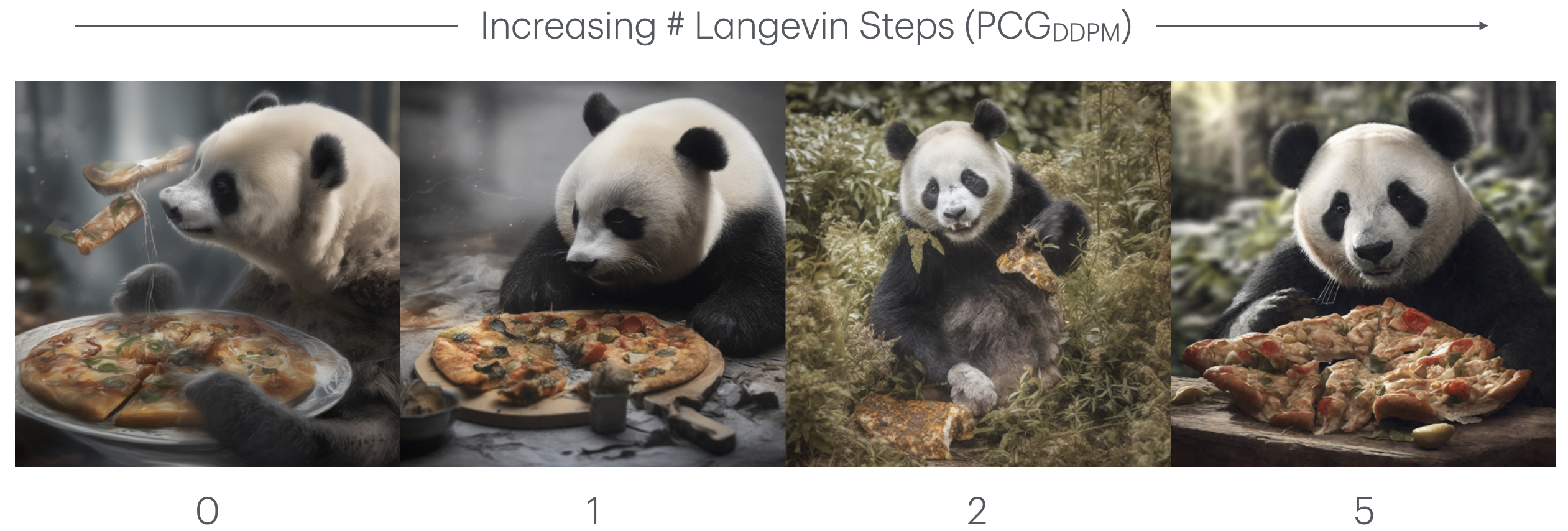}
    \caption{{\bf Effect of Langevin Dynamics}. PCG generations with $\gamma=1$ (no guidance) fixed and number of Langevin steps $K$ varied. The prompt is ``photograph of a panda eating pizza''. Increasing the number of Langevin steps can qualitatively improve image quality, even without guidance.}
    \label{fig:panda}
\end{figure}

\section{Algorithms}
\algorithmstyle{ruled}
\begin{algorithm}
\newcommand\mycommfont[1]{\normalfont\textcolor{black}{#1}}
\SetCommentSty{mycommfont}
\DontPrintSemicolon
\SetNoFillComment
\caption{$\pcg_\ddim$, explicit}\label{alg:pcg-explicit}
\SetKwInput{KwData}{Constants}
\SetKwComment{Comment}{$\triangleright$\ }{}
\KwIn{Conditioning $c$, guidance weight $\gamma \geq 0$}
\KwData{$\{\alpha_t\}, \{\bar{\alpha}_t\}, \{\beta_t\}$ from \citet{ho2020denoising}}
$x_1 \sim \cN(0, I)$ \;
\For{$(t=1-\dt;~ t \geq 0;~ t \gets t- \dt)$}{
$\eps, \eps_c := \textsf{NoisePredictionModel}(x_{t+\dt}, c)$\;
$\hat{x}_0 := (x_{t+\dt} - \sqrt{1-\bar{\alpha}_{t+\dt}}\eps_c) / \sqrt{\bar{\alpha}_{t+\dt}}$ \;
$x_t := \sqrt{\bar{\alpha}_{t}} \hat{x}_0 + \sqrt{1-\bar{\alpha}_{t}} \eps_c$ \Comment*{DDIM step on $p_{t}(x | c)$}
    \For{$k = 1,\dots K$}{
        $x_t \gets x_t  - \frac{\beta_t}{2\sqrt{1-\bar{\alpha}_t}}\left((1-\gamma)\eps  + \gamma\eps_c \right) + \sqrt{\beta_t} \eta$ \Comment*{Langevin dynamics on $p_{t, \gamma}(x|c)$}
    }
}
\Return $x_0$
\end{algorithm}

\end{document}